\documentclass{article} 
\usepackage{iclr2025_conference,times}

\usepackage{amsmath,amsfonts,bm}

\def\eqref#1{equation~\ref{#1}}

\def\1{\bm{1}}

\DeclareMathAlphabet{\mathsfit}{\encodingdefault}{\sfdefault}{m}{sl}
\SetMathAlphabet{\mathsfit}{bold}{\encodingdefault}{\sfdefault}{bx}{n}

\usepackage{hyperref}
\usepackage{url}
\usepackage{graphicx} 
\usepackage{tcolorbox}
\usepackage{amsthm}
\usepackage{adjustbox}
\usepackage{booktabs} 

\title{Just say what you want: only-prompting self-rewarding online preference optimization}

\author{Ruijie Xu$^{\text{1}}$\thanks{This work is completed during Ruijie Xu's internship at ByteDance Inc.}, Zhihan Liu$^{\text{2}}$, Yongfei Liu$^{\text{3}}$, \\ 
\textbf{Shipeng Yan$^{\text{3}}$, Zhaoran Wang$^{\text{2}}$, Zhi Zhang$^{\text{3}}$, Xuming He$^{\text{1}}$} \\
$^{\text{1}}$ShanghaiTech University $^{\text{2}}$Northwestern University $^{\text{3}}$ ByteDance Inc \\
$^{\text{1}}$\texttt{\{xurj2022, hexm\}@shanghaitech.edu.cn}\\
$^{\text{2}}$\texttt{zhihanliu2027@u.northwestern.edu, zhaoranwang@gmail.com}\\
$^{\text{3}}$\texttt{\{liuyongfei.0314, yanshipeng, zhangzhi.joshua\}@bytedance.com}
}

\iclrfinalcopy 
\begin{document}

\maketitle

\begin{abstract}
We address the challenge of online Reinforcement Learning from Human Feedback (RLHF) with a focus on self-rewarding alignment methods. In online RLHF, obtaining feedback requires interaction with the environment, which can be costly when using additional reward models or the GPT-4 API. Current self-rewarding approaches rely heavily on the discriminator's judgment capabilities, which are effective for large-scale models but challenging to transfer to smaller ones.
To address these limitations, we propose a novel, only-prompting self-rewarding online algorithm that generates preference datasets without relying on judgment capabilities. Additionally, we employ fine-grained arithmetic control over the optimality gap between positive and negative examples, generating more hard negatives in the later stages of training to help the model better capture subtle human preferences. Finally, we conduct extensive experiments on two base models, Mistral-7B and Mistral-Instruct-7B, which significantly bootstrap the performance of the reference model, achieving 34.5\% in the Length-controlled Win Rates of AlpacaEval 2.0.
\end{abstract}

\section{Introduction}

Reinforcement Learning from Human Feedback (RLHF) is a prevalent technique for Large Language Model (LLM) alignment, ensuring models adhere to human preferences, produce useful and truthful responses, and prevent harmful ones \citep{stiennon2020learning, ouyang2022training, christiano2017deep}. Current RLHF methods are classified into online and offline approaches \citep{rafailov2024direct, xiong2024iterative, meng2024simpo}. The primary distinction is that online methods involve real-time interaction with the environment for feedback, while offline methods rely on pre-existing datasets without environmental interaction. Due to distribution shift challenges in offline settings \citep{liu2023statistical}, our focus is on online approaches.

However, many existing online methods necessitate an auxiliary reward model \citep{dong2024rlhf, zhang2024self} or a potent LLM, like GPT-4 \citep{xiong2024iterative, guo2024direct}, for evaluating responses produced by the current policy. This evaluation entails assessing numerous samples, resulting in a substantial cost when using the GPT-4 API. Moreover, the considerable expense of human annotation makes training an efficient reward model a costly endeavor.

Based on this, researchers have investigated LLM alignment using self-rewarding methods, such as those presented in \citet{yuan2024self,chen2024self,liu2024direct}. However, most self-rewarding approaches still depend on discriminators. For instance, \citet{yuan2024self} propose a policy that serves as both an actor and a judge, using a prompt to score its responses. Despite this, discriminators continue to rely on large-scale models like LLaMA-70B, which is inefficient in terms of both memory usage and inference speed. Moreover, smaller language models lack the necessary discriminative power \citep{jiang2024self}, making it challenging to effectively evaluate preference datasets. We adopt the approach presented in \cite{yuan2024self} and apply it to Mistral-7B, as opposed to their utilization of LLaMA-70B. As depicted in Figure \ref{fig:sr}, the response score distribution is considerably narrow, with more than 80\% of responses receiving a rating of 4 out of 5. This observation highlights the limited judgment capacity of smaller models in distinguishing between high-quality and low-quality responses, thereby hindering their effectiveness in generating preference datasets.

Conversely, we find that utilizing generative capabilities enables the generation of both positive and negative examples. By defining two distinct prefixes—one representing high-quality and the other low-quality outcomes—the model can be guided to produce corresponding results. Our theoretical framework and experimental results support this notion, demonstrating that responses generated with different prefixes display a preference gap that can be harnessed for preference learning.



Furthermore, our self-generation approach enables precise arithmetic control over the optimality gap between positive and negative examples. As training progresses and the policy effectively distinguishes significantly different pairs, we narrow this gap. This promotes the generation of more challenging pairwise examples (hard negatives), compelling the policy to discern subtle distinctions between positives and negatives, ultimately leading to better alignment with intricate human preferences.



In order to validate the efficacy of our proposed approach, we perform comprehensive experiments on two foundational models, namely Mistral-7B and Mistral-Instruct-7B. We employ well-established evaluation metrics, AlpacaEval 2.0 and MT-Bench, to assess the performance. The experimental outcomes reveal that our technique substantially surpasses multiple baseline methods in both AlpacaEval 2.0 and MT-Bench assessments.

The contributions of our method are summarized in three-folds:

\begin{enumerate}
\item We propose a novel online only-prompting self-rewarding alignment framework that leverages generation capability to create preference datasets without needing a discriminator.

      \item We generate fine-grained control over the optimality gap between positive and negative examples, creating more hard negative cases in later training stages to better align the model with complex human preferences.

\item Our method achieves significant performance improvements on AlpacaEval 2.0 and MT-Bench with Mistral-7B and Mistral-Instruct-7B.
\end{enumerate}

\begin{figure}[!t]
    \centering
     \includegraphics[width=0.5\linewidth]{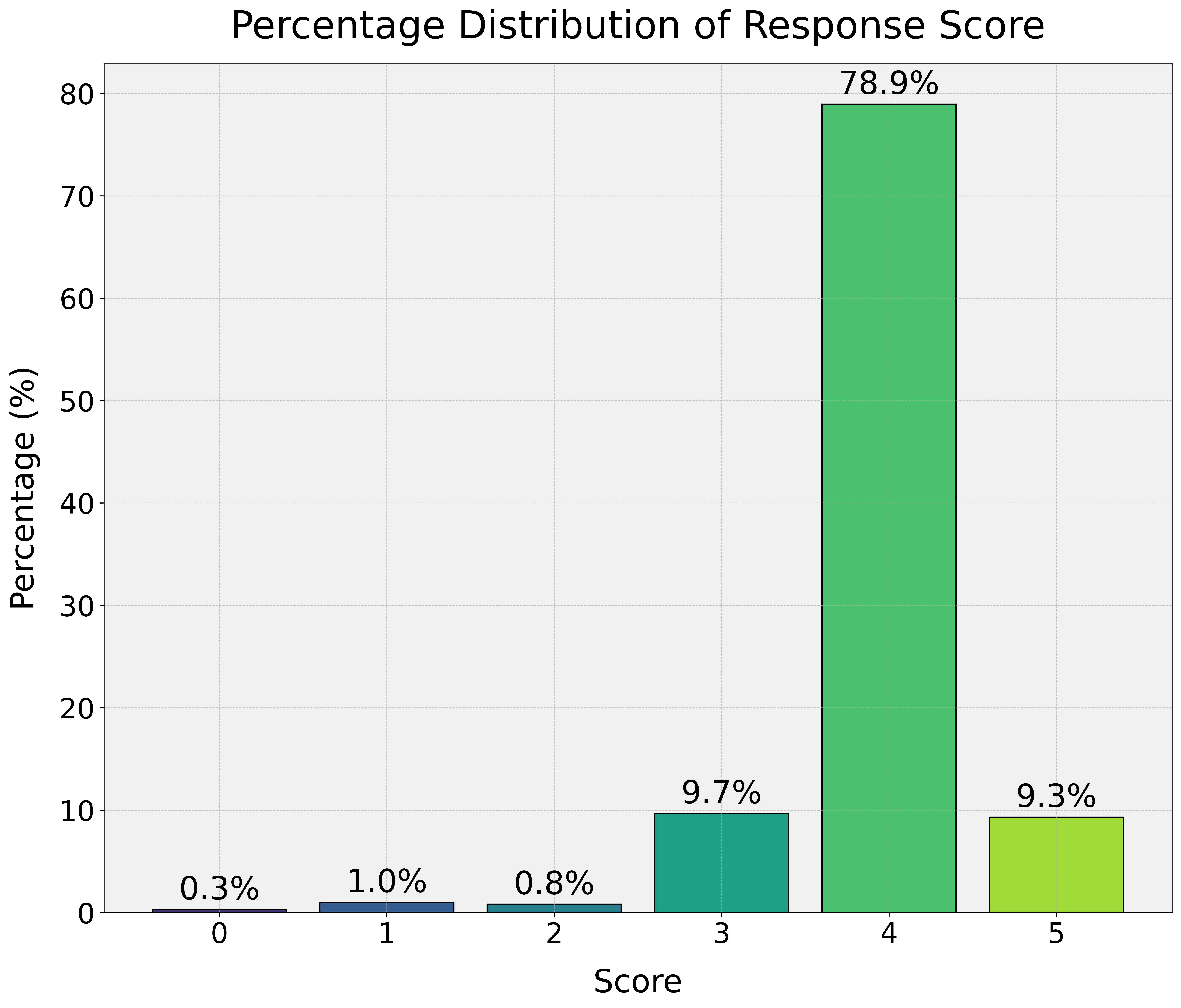}
    \caption{Distribution of scores for responses generated by the self-rewarding algorithm \citep{yuan2024self} on the Mistral-7B model. The prompts come from Ultralfeedback \citep{cui2023ultrafeedback}, totaling approximately 60k. About 80\% of the responses are rated with a score of 4.}
    \label{fig:sr}
\end{figure}

\section{Related work}
\label{relatedwork}
\subsection{Reinforcemant learning from human feedback}
RLHF \citep{bai2022training,christiano2017deep,ouyang2022training,ziegler2019fine} has recently proven crucial in developing state-of-the-art LLMs like ChatGPT \citep{achiam2023gpt} , Gemini \citep{team2023gemini}, and Claude. The standard RLHF framework for LLM alignment is proposed by \citet{ouyang2022training}. They train a reward model on a dataset of human preferences and then fine-tuned a pretrained LLM to maximize the reward from this reward model using the Proximal Policy Optimization (PPO) algorithm \citep{schulman2017proximal, zheng2024toward, liu2023reason}. 
However, PPO-style algorithms are characterized by instability, sample inefficiency, and a high demand for careful hyperparameter tuning \citep{engstrom2020implementation}. Fitting a high-quality reward model also requires substantial human-labeled data. Consequently, these factors lead to prohibitively high computational costs for PPO-based RLHF methods. 

\subsection{Preference optimization}
\subsubsection{Offline Preference optimization}
Therefore, recent research on RLHF has explored several alternatives to PPO-based methods, with direct preference optimization (DPO) \citep{rafailov2024direct} emerging as the most widely used option. DPO is an RL-free algorithm for training language models from preferences. By optimizing a policy using a straightforward binary cross-entropy objective, DPO eliminates the need for reward model training and directly fits an implicit reward model to the preference data. Compared to PPO in RLHF, DPO offers greater stability and reduced computational demands. Several variants of the direct preference learning approach have been proposed \citep{meng2024simpo,azar2024general,ethayarajh2024kto,xu2024contrastive, wu2024self}, each aiming to address additional challenges of direct preference learning from different perspectives.
\subsubsection{Iterative Preference optimization}
Although DPO-style methods have made significant progress, the preference datasets used in DPO typically consist of responses generated by different LLMs. As a result, the policy model does not receive feedback on its own generations during training, leading to a substantial distribution shift between the policy that generated the dataset and the aligned policy \citep{xu2024dpo}. To address this, further studies have extended the approach to an iterative training setup, continuously updating the reference model for optimization and sampling the current policy to generate preference datasets \citep{kim2024sdpo, xiong2024iterative, dong2024rlhf, rosset2024direct}.  However, many of these methods still require an additional reward model or a powerful LLM, such as GPT-4, to label preference data.

\subsection{Reward-free alignment}
Due to the high cost of human annotation, training an effective reward model is very expensive. Some methods attempt to generate preference datasets without external knowledge, such as \citet{yuan2024self}, which use a single model for both policy and judgment. However, these approaches often rely on powerful models like LLaMA-70B, making it challenging for smaller models to achieve the necessary judgment capabilities. \citet{jiang2024self} argue that discrimination is not reliably better than generation, and in fact, performs worse. This suggests that discrimination is a higher-order ability compared to generation. Our approach bypasses discrimination by directly allowing the model to generate preference datasets.
\cite{liu2024direct} propose designing contrastive prompts to generate preference datasets, but their approach is limited to offline settings, and the quality of the generated preference data cannot be guaranteed. 
Self-play finetuning (SPIN) \citep{chen2024self} relies solely on seed SFT data by pairing a ground truth response as a positive sample with a model-generated response as a negative. However, this method assumes that the currently generated responses are always inferior to the original data, which may not accurately reflect reality. Our method is an online, only-prompting self-rewarding preference optimization approach that does not require an additional discriminator for judgment. Instead, it efficiently samples high-quality preference datasets. 

\section{Method}

\begin{figure}[!t]
    \centering
     \includegraphics[width=\linewidth]{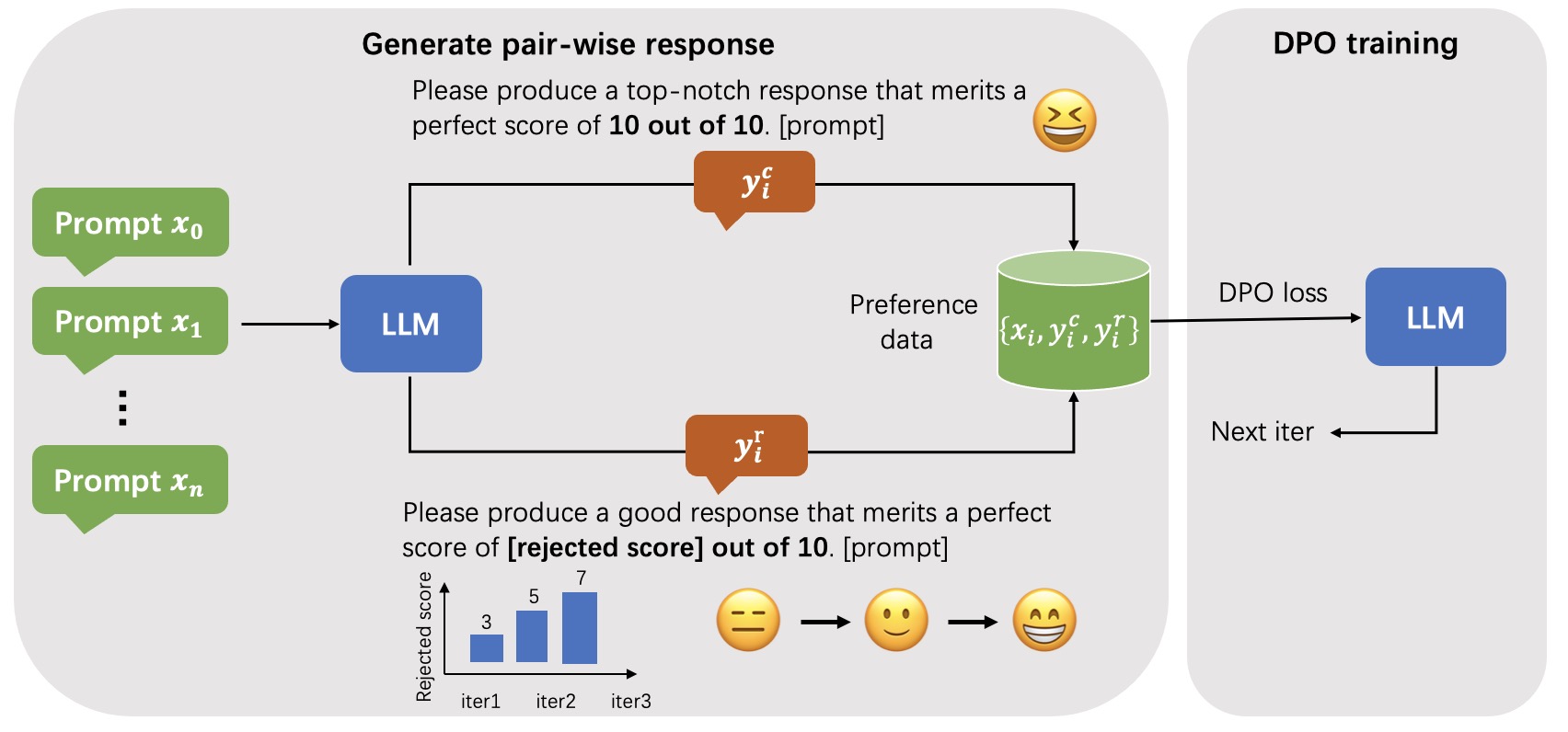}
    \caption{Our method consists of two parts: generating the preference dataset and conducting DPO training. When generating the preference dataset, the input is the prompt \( x_i \). To generate the chosen response \( y^c_i \), we prepend a chosen prefix to \( x_i \), and to generate the rejected response \( y^r_i \), we prepend a rejected prefix to \( x_i \). The final preference dataset is composed of \( \{x_i, y^c_i, y^r_i\} \). This dataset will be used for the current round of DPO training. The trained model from this round will serve as the reference model for the next round.
    }
    \label{fig:frame}
\end{figure}

In this section, we first introduce an overview of our method in Section \ref{sec:over}. We then describe the initialization of our method in Section \ref{sec:ini}, followed by the process of generating the preference dataset in Section \ref{sec:gen}. Subsequently, we mathematically prove the existence of a quality gap between the chosen and rejected response, which can be utilized for DPO training in Section \ref{sec:the}. Finally, we present our iterative training strategy in Section \ref{sec:iter}.


\subsection{Overview}\label{sec:over}
In self-rewarding RLHF tasks, the policy does not interact with an external environment to receive feedback. Current methods rely heavily on a discriminator; however, smaller models have weaker judgment abilities, making them less suitable for such tasks. Moreover, repeatedly sampling for a given prompt leads to inefficient sampling. 
To address the reliance on discriminators and enhance effectiveness on small models, we propose a novel online only-prompting self-rewarding alignment framework. The key ideas of our method include two aspects: (1) directly leveraging the generation capability to create preference datasets without the need for a discriminator; (2) generating fine-grained control over the optimality gap between positive and negative examples, which allows for the creation of more challenging negative cases in later training stages, thereby better aligning the model with complex human preferences. An overview of our framework is depicted in Figure \ref{fig:frame}.

\subsection{Initialization}\label{sec:ini}
In self-rewarding LLM alignment, we essentially align it without relying on any external preference feedback. 
Similar to \cite{yuan2024self}, our method first assumes access to a base pretrained language model and a amount of human-annotated seed data $\{x_{\text{sft}}, x_{\text{dpo}}\}$. Our seed data consists of two parts: one part is instruction-following data used for SFT, and the other part is preference data used for offline preference optimization. We find that applying our method in a bootstrap manner on a model trained with a seed preference dataset can lead to greater performance improvement. 
This is likely because a model that undergoes preference learning gains a better understanding of human preferences, resulting in stronger generative capabilities. Our initialization process aligns with the offline preference optimization procedure, as detailed below:
\begin{enumerate}
\item  Given a pre-trained model \(\theta\), we fine-tune \(\theta\) on a supervised fine-tuning (SFT) dataset \(\mathcal{D}_{\text{sft}} = \{(x_i, y_{\text{target}_i})\}_{i=1}^{N}\) using the following SFT loss function to obtain the model \(\theta_{\text{SFT}}\)
\begin{equation}
    \mathcal{L}_{\text{sft}}(\theta) = \mathbb{E}_{(x, y) \sim \mathcal{D}_{\text{sft}}} \left[ -\log \pi_\theta(y \mid x) \right],
\end{equation}
where  \(\pi_\theta\) represents the policy parameterized by \(\theta\).
\item We apply Direct Preference Optimization (DPO) to the offline preference dataset \(\mathcal{D}_{\text{DPO}}\). The associated loss function is given by

\begin{equation}
\mathcal{L}_{\text{DPO}}(\pi_\theta; \pi_{\text{sft}}) = -\mathbb{E}_{(x, y_w, y_l) \sim \mathcal{D}_{\text{DPO}}} \left[ \log \sigma \left( \beta \log \frac{\pi_\theta(y_w \mid x)}{\pi_{\text{ref}}(y_w \mid x)} - \beta \log \frac{\pi_\theta(y_l \mid x)}{\pi_{\text{ref}}(y_l \mid x)} \right) \right],
\end{equation}
where \(\sigma\) denotes the sigmoid function, \(\pi_{\text{ref}}\) denotes the reference policy, and \(y_l\) and \(y_w\) represent the rejected response and chosen response, respectively.

Note that this loss term can represent either the standard DPO loss or an improved version. Our method is applicable to any DPO-type method.
\end{enumerate}

\subsection{Generate preference dataset}
\label{sec:gen}

To obtain the preference dataset, we explicitly define a response score for each response, with higher scores indicating better quality. We design two different prompts to leverage the generative ability for eliciting both positive and negative examples. The current policy is directly instructed to generate both high-scoring and low-scoring responses, with the high-scoring responses labeled as \textit{chosen} and the low-scoring responses labeled as \textit{rejected}.
Specifically, given an input instruction \(\{x_i\}_{i=1}^N\), the input for generating a chosen response is \(\{(p_c, x_i)\}_{i=1}^N\), while the input for generating a rejected response is \(\{(p_r, x_i)\}_{i=1}^N\). The prefixs \(p_c\) and \(p_r\) are defined as follows:

\begin{tcolorbox}
\textbf{Choosen prefix $p_c$}

Please produce a top-notch response that merits a perfect score of 10 out of 10. [prompt]

\textbf{Rejected prefix $p_r$}

Please produce a good response that merits a perfect score of [rejected score] out of 10. [prompt]
\end{tcolorbox}

By designing such pair-wise prompts, the model generates two responses, \(\{y^c_i\}_{i=1}^N\) and \(\{y^r_i\}_{i=1}^N\). We observe that this straightforward design enables the language model to interpret distinct scoring prompts effectively, thereby producing finely differentiated chosen and rejected responses that align with human preferences.

\subsubsection{Theoretical demonstration}\label{sec:the}
In this section, we provide a theoretical proof that the chosen responses generated in Section \ref{sec:gen} are of higher quality than the rejected responses, making them suitable for constructing preference datasets for subsequent rounds of preference optimization. For simplicity, we denote the input instruction as \(x\), the generated response as \(y \in \mathcal Y\), where $\mathcal Y$ is response space, and the reward function $f$ as  $f:\mathcal{X}\times\mathcal{Y}\mapsto \mathbb{R} $, which evaluates the quality of response \(y\) given instruction \(x\). The policy distribution of the language model, \(\pi(y \mid x)\), determines the probability of generating a specific response \(y\) based on instruction \(x\). 
We consider the policies that share the same support as the reference policy $\pi_{\text{ref}}$, we take a policy class $\Pi$ as 
\begin{equation}
 \Pi =
\left\{\; \pi : \mathcal X \mapsto \Delta(\mathcal Y) \mid \text{Supp}(\pi(\cdot|x)) \subseteq \text{Supp}(\pi_{\text{ref}}(\cdot|x)), \; \forall x \in \mathcal X \right\},
\end{equation}
wher $\text{Supp}(p)$ denotes the support of a probability density function $p$.

\newtheorem{assumption}{Assumption}

\begin{assumption}[The range of $f(x,y)$ is bounded]
\label{assump:example}
For all \( x \) and \( y \), the function \( f(x, y) \) satisfies the bound \( 0 \leq f(x, y) \leq R_{\text{max}} \), where \( R_{\text{max}} \) is a constant.
\end{assumption}

\begin{assumption}[The reward score $r$ given instruction $x$ and response $y$ follows a unimodal distribution]
\label{assump:2} Given any \( x \) and \( y \), the conditional probability \( p(r \mid x, y)  \propto \exp\left(-\gamma \|f(x, y) - r\|^2_\alpha\right) \), where \( \| \cdot \|_\alpha \) denotes the \(\alpha\)-norm. For example, setting \(\alpha = 2\) corresponds to a Gaussian distribution. The parameter \(\gamma\) controls the steepness of the corresponding distribution. As \(\gamma\) increases, the peak of corresponding distribution becomes steeper, leading to a more concentrated distribution. 
\end{assumption}

\newtheorem{proposition}{Proposition}

\begin{proposition}[Quality gap between responses]
\label{prop:example}
Given the instruction $x$ and different reward scores, the policy can generate responses of varying quality.
We denote two different policies as $\pi_{\text{good}}$ and $\pi_{\text{bad}}$, which have the following forms
\begin{equation}
    \pi_{\text{good}} = \pi(\cdot|\cdot,r=R_{\text{max}}),
\end{equation} 
\begin{equation}
    \pi_{\text{bad}} = \pi(\cdot|\cdot,r=R_{0}), R_0 < R_{\text{max}}.
\end{equation}
\end{proposition}
There is a quality difference between \(\pi_{\text{good}}\) and \(\pi_{\text{bad}}\), with the quality of \(\pi_{\text{bad}}\) being slightly lower than that of \(\pi_{\text{good}}\). Therefore, in DPO training, responses generated by \(\pi_{\text{bad}}\) can serve as the rejected responses, while those generated by \(\pi_{\text{good}}\) can serve as the chosen responses. 

\newtheorem{theorem}{Theorem}
\newtheorem{lemma}{Lemma}
\begin{proof}
Proof of Proposition.\ref{prop:example}.
By Bayes' theorem, we have
\begin{equation}
\pi(y|x, r) = \frac{p(r|x, y) \, p(y|x)}{p(r|x)},
\end{equation}

Thus, we obtain
\begin{align}
\pi(y|x, r) &\propto p(r|x, y) \, p(y|x)\\
            &\propto \exp\left(-\gamma{\|f(x, y) - r\|}_\alpha^2\right)  \cdot\pi_{\text{ref}}(y|x).
\end{align}

From \citet{ouyang2022training, xiong2024iterative, liu2024provably}, the standard RLHF objective function is given by
\begin{equation}\label{equ:rlhf}
J(\pi) = \mathbb{E}_{x\sim D, y\sim \pi(\cdot|x)}\bigr[f(x,y) -\beta\cdot \text{KL}(\pi(\cdot\mid x)\|\pi_{\text{ref}}(\cdot\mid x))\bigr].
\end{equation}

Substituting Assumption.\ref{assump:2}, our RLHF objective can be expressed as 
\begin{equation}
J(\pi) = \mathbb{E}_{x\sim D, y\sim \pi(\cdot|x)}\bigr[-\gamma\|f(x,y)-R_{\text{max}}\|_\alpha^2 -\beta\cdot \text{KL}(\pi(\cdot\mid x)\|\pi_{\text{ref}}(\cdot\mid x))\bigr].
\end{equation}

\begin{lemma}[Oracle optimal KL-regularized policy]
\label{lemma:example}
Given any reward model $p \in \mathcal{R}$, the optimal policy $\pi_p$ to
\begin{equation}
    \max_{\pi\in\Pi} \{\mathbb{E}_{x\sim D, y\sim \pi(\cdot|x)}\bigr[-\gamma\|f(x,y)-R_{\text{max}}\|_\alpha^2 -\beta\cdot \text{KL}(\pi(\cdot\mid x)\|\pi_{\text{ref}}(\cdot\mid x))\bigr]\}
\end{equation}

is given by 
\begin{equation}
    \pi_p(\cdot|x)=\frac{1}{Z_p(x)}\cdot\pi_{\text{ref}}\cdot\exp{(\beta^{-1}p(x,\cdot)),Z_p(x)=\int _{a\in\mathcal{A}}\exp(\beta^{-1}p(x,a))\mathrm d\pi_{\text{ref}}(a|x),}
\end{equation}
and correspondingly the optimal value of Equation \ref{equ:rlhf} is given by $\mathbb{E}_{x \sim d_0}\left[\beta \cdot \log(Z_p(x))\right]$.
\end{lemma}
\begin{proof}
The proof of Lemma.\ref{lemma:example} can be derived directly from Proposition 7.16 and Theorem 15.3 of \citet{zhang2023mathematical}.
\end{proof}
Substituting this, we obtain
\begin{align}\label{equ:9}
    J(\pi_{\text{good}}) &= \max_{\pi} J(\pi).      
\end{align}
By Assumption \ref{assump:2} and Equation \ref{equ:9}, we can obtain
\begin{align}
J(\pi_{\text{good}}) &\geq J(\pi_{\text{bad}}).
\end{align}

Since $\beta > 0$, $\pi_{\text{good}} \neq \pi_{\text{bad}}$, and $J(\pi)$ is a strongly convex function with respect to $\pi$, we obtain
\begin{align}\label{equ:11}
J(\pi_{\text{good}}) > J(\pi_{\text{bad}}).
\end{align}

From Equation \ref{equ:11}, we observe a quality gap in the objective function between \( \pi_{\text{good}} \) and \( \pi_{\text{bad}} \). This indicates that responses sampled from \( \pi_{\text{good}} \) are likely to exhibit higher quality compared to those from \( \pi_{\text{bad}} \), as \( \pi_{\text{good}} \) generally yields a lower loss.

\end{proof}

\subsection{Iterative training strategy}\label{sec:iter}
We apply an iterative framework to optimize the current policy, dividing the optimization process into $M$ iterations. For the first iteration, the reference model is the model $\pi_\theta$ obtained after initialization. For subsequent iterations, the reference model is the model from the end of the previous iteration.

We believe that in the early stages of training, when the policy has not yet aligned with human preferences, pair-wise datasets with a large gap between the chosen and rejected responses are still beneficial. However, as training progresses and the policy increasingly learns human preferences, it becomes easier to distinguish between chosen and rejected responses with a large gap. Feeding the model such simple data at this stage no longer contributes to further learning improvements. To ensure that the model continues to learn complex human preferences rather than merely becoming proficient with basic distinctions, we reduce the gap between pair-wise data in the later stages of training. Specifically, we progressively increase the rejected score of the rejected prefix $p_r$. For instance, in the first iteration, the score is set to 3, in the second iteration to 5, and in the third iteration to 9.
This encourages the generation of more challenging pairwise examples (hard negatives) in the later stages, forcing the policy to detect subtle differences between negatives and positives, thereby aligning with complex human preferences.

\section{Experiment}
\subsection{Expreiment setup}

\paragraph {\textbf{Models and training settings.}} We conduct our experiments on Mistral-7B \footnote{https://huggingface.co/mistralai/Mistral-7B-v0.1} and Mistral-Instruct-7B \footnote{https://huggingface.co/mistralai/Mistral-7B-Instruct-v0.2}. Our training code is based on the alignment-handbook \citep{Tunstall_The_Alignment_Handbook}. We use vLLM \citep{kwon2023efficient} to sample responses.   
Following SimPO \citep{meng2024simpo}, we adopt the training pipeline outlined in their work, utilizing their SFT model. For the base version, we align with Zephyr \citep{tunstall2023zephyr}, and for the instruct version, we use mistralai/Mistral-7B-Instruct-v0.2 as the SFT model. For the initialization dataset, we remain consistent with SimPO. For the base setup, we use the UltraFeedback dataset \citep{cui2023ultrafeedback}. For the instruct version, we use the datasets\footnote{https://huggingface.co/datasets/princeton-nlp/mistral-instruct-ultrafeedback} published by SimPO. Motivated by SimPO, for the second phase of initialization, we apply two tricks from their work: length-normalized reward formulation and target reward margin. We find that these techniques achieve better results compared to naive DPO. Our training parameters are largely consistent with SimPO. However, we observed that as the number of iterations increases, it is necessary to reduce the learning rate. For iterations 2 and 3, we use learning rates of 1e-7 and 1e-8, respectively. In our experimental setup, \( M \) is set to 3.

\paragraph {\textbf{Evaluation benchmarks. }}
We apply the most widely recognized public instruction-following benchmarks, MT-Bench \citep{zheng2023judging} and AlpacaEval 2.0 \citep{alpaca_eval}, to evaluate our method. These benchmarks assess the models' capabilities in handling a diverse range of conversational queries and are well-regarded within the community. Both benchmarks necessitate a judgment model (GPT-4) for scoring. To ensure a fair comparison with previous baselines, we employ GPT-1104-preview \footnote{https://openai.com/api/}, a high-performance variant, for the judgment. 

\paragraph {\textbf{Baseline}}
Our baseline consists of two components: offline preference optimization methods and self-rewarding online optimization methods. For the offline methods, we compare RRHF \citep{yuan2024rrhf} and SLiC-HF \citep{zhao2023slic}, which apply ranking loss; IPO \citep{azar2024general}, which avoids the pair-wise reward assumption; CPO \citep{xu2024contrastive} and ORPO \citep{hong2024orpo}, which incorporate the SFT objective into the loss; KTO \citep{ethayarajh2024kto}, which does not require training on paired datasets; and R-DPO \citep{park2024disentangling} and SimPO \citep{meng2024simpo}, which use a length-normalized reward formulation. For self-rewarding online optimization methods, we compare self-rewarding(SR) \citep{yuan2024self}, which use a single model for both policy and judgment. The prompts used for judgment and scoring are provided in the Appendix.

\begin{table}[!t]
\centering
\caption{Comparison of Baseline for AlpacaEval 2.0 and MT-Bench. LC and WR refer to Length-Controlled Win Rates and Win Rates, respectively. SR denotes self-rewarding \citep{yuan2024self}, which we have reproduced on Mistral-7B and Mistral-Instruct-7B. Iter1-Iter3 represent the first to third iterations. In our method, we use the SimPO \citep{meng2024simpo}'s DPO loss for training. 
The results of the offline preference optimization algorithm are copied from \citet{meng2024simpo}}.
\label{tab:main}
\begin{adjustbox}{width=\textwidth}
\begin{tabular}{@{}llcclclcclc@{}}
\toprule
             &  &      & \textbf{Mistral-7B}      & \multicolumn{1}{c}{\textbf{}} & \textbf{}         &           &  &    \textbf{Mistral-Instruct-7B}           & \multicolumn{1}{c}{} &                   \\ \midrule
             &  & \multicolumn{2}{c}{\textbf{AlpacaEval 2.0}} & \textbf{}                     & \textbf{MT-Bench} & \textbf{} & \multicolumn{2}{c}{\textbf{AlpacaEval 2.0}}   &                      & \textbf{MT-Bench} \\ \cmidrule(lr){3-4} \cmidrule(lr){6-6} \cmidrule(lr){8-9} \cmidrule(l){11-11} 
             &  & LC (\%)                     & WR (\%)                &                               &        &           & LC (\%)                           & WR (\%)            &                      &        \\ \midrule
SFT          &  & 8.4                     & 6.2              &                               & 4.8               &           & 17.1                         & 14.7          &                      & 6.2               \\ \midrule
RRHF          &  & 11.6                    & 10.2             &                               & 5.4               &           & 25.3                         & 24.8          &                      & \textbf{6.5}               \\
SLiC-HF          &  & 10.9                    & 8.9             &                               & 5.8               &           & 24.1                         & 24.6          &                      & \textbf{6.5}               \\

DPO          &  & 16.8                    & 12.8             &                               & 5.8               &           & 26.8                         & 24.9          &                      & 6.3               \\
IPO          &  & 11.8                    & 9.4              &                               & 5.5               &           & 20.3                         & 16.2          &                      & 6.4               \\
CPO          &  & 9.8                   & 8.9             &                               & 5.4               &           & 23.8                         & 28.8          &                      & 6.3               \\
KTO          &  & 13.1                    & 9.1              &                               & 5.4               &           & 23.6                         & 17.9          &                      & 6.4               \\
ORPO         &  & 14.7                    & 12.2             &                               & 5.8               &           & 24.5                         & 20.8          &                      & 6.4               \\
R-DPO        &  & 17.4                    & 12.8             &                               & 5.9               &           & 24.5                         & 16.1          &                      & 6.2               \\
SimPO        &  & 21.5                    & 20.1             &                               & \textbf{6.0}      &           & 30.4                   & 34.2          &                      & 6.4               \\ \midrule
SR-iter1     &  &   19.9                   &  17.9              &                                       &    5.6               &         &            26.7                    &   25.4            &                      &     6.1              \\
SR-iter2     &                      &    20.8              &        18.5                       &                   &   5.5        &      &     27.9          &            26.8          &   &   6.3             \\
SR-iter3     &  &   19.7                       &   17.6               &                               &       5.5            &                             &        27.0         &       26.0                     &    & 6.3               \\ \midrule
Ours-iter1 &  & 23.4                    & 21.6             &                               & 5.8               &           & 32.5                         & 37.5          &                      & 6.4               \\
Ours-iter2 &  & 24.2                    & 23.5             &                               & \textbf{6.0}      &           & 33.9                         & 38.7          &                      & \textbf{6.5}               \\
Ours-iter3 &  & \textbf{25.9}           & \textbf{24.9}    &                               & \textbf{6.0}      &           & \textbf{34.5}                & \textbf{42.0} &                      & 6.4      \\ \bottomrule
\end{tabular}

\end{adjustbox}
\end{table}

\subsection{Result}

As presented in Table \ref{tab:main}, our method shows significant improvements on AlpacaEval 2.0 and comparable performance on MT-Bench. It is evident that by the first iteration, our method significantly outperforms the current state-of-the-art offline method, SimPO. On AlpacaEval 2.0, our method achieves nearly a 2\% improvement over SimPO for both Mistral-7B and Mistral-Instruct-7B in the first iteration. For self-rewarding methods, the results are suboptimal due to the limited judgment capability of smaller models, which hinders their ability to rank responses effectively, resulting in noisy preference data. As the number of iterations increases, AlpacaEval 2.0 scores gradually improve, ultimately surpassing SimPO by \textbf{4\%}. For Mistral-7B and Mistral-Instruct-7B, this yields scores of \textbf{25.9\%} and \textbf{34.5\%}, respectively. 
For MT-Bench, we find that we achieve results comparable to SimPO. As \cite{meng2024simpo} mentions, minor differences between methods on MT-Bench are possibly a result of randomness, influenced by the limited scale of its evaluation data and the use of a single-instance scoring protocol.

\subsection{Ablation study}
\paragraph {\textbf{Arithmetic control of the optimality gap}} To analyze the effectiveness of each component, we conduct extensive experiments on Mistral-7B. We examine the impact of arithmetic control of the optimality gap by fixing the rejected response scores across all iterations and performing a set of experiments, as shown in Table \ref{tab:score}.
From the table, it is evident that removing arithmetic control results in no improvement in performance for iter2 and iter3, with AlpacaEval 2.0 LC remaining around 23\% and MT-Bench at 5.9. However, with arithmetic control, the final AlpacaEval 2.0 score increases to 25.9\% and MT-Bench to 6.0. This demonstrates that arithmetic control of the optimality gap is crucial. In the later stages of training, the model requires more challenging hard negatives. Exposure to easier cases with large gaps is not effective for learning. Human preferences are complex and difficult to model with simple methods. Gradually reducing the gap during training helps the model focus on hard negatives, allowing it to better capture the nuances of human preferences.
\paragraph {\textbf{Inference on chosen prompt}}
To verify whether the gain brought by the chosen prefix can be directly obtained from inference, we apply the chosen prompt directly to the reference model during the first round of inference. The results show a slight improvement in AlpacaEval 2.0, but it still falls short of the results obtained after the first round of training. For MT-Bench, the results decrease. Upon examining the model outputs, we find that adding the chosen prefix during inference leads to outputs that include additional information beyond the response, such as ``Okay, here is a 10-score answer'' or ``The 10-score answer is as follows.'' This extra information results in GPT-4 assigning lower scores during evaluation. When generating the preference dataset, we specifically use regular expressions to remove responses containing redundant information, thereby avoiding issues related to extra information.

\paragraph {\textbf{Analysis experiment}}

We evaluate the chosen response and rejected response for each round using an open reward model\footnote{https://huggingface.co/berkeley-nest/Starling-RM-7B-alpha} to see if the results aligned with our expectations. 
We analyze this from two perspectives: first, whether the same model could generate fine-grained quality assessments based on the responses. From Figure \ref{fig:bar2}, it can be observed that the same model generates responses of varying quality with different prefixes. Higher prefix scores lead to higher quality responses, and lower prefix scores result in lower quality responses. Such post hoc experiments demonstrate that we can design different prefix scores to generate preference datasets of varying quality.
Additionally, we observe that as iterations increase, both the reward scores for chosen and rejected responses rise, with the gap between them gradually narrowing, as shown in Figure \ref{fig:bar3}. This could be due to two possible reasons: firstly, the model's capability improves over time, becoming more aligned with human preferences, and thus generating higher quality responses; secondly, we reduce the prefix score difference between chosen and rejected responses, leading to more negatives being generated in the later stages of training. We discuss the advantages of this design in Section \ref{sec:iter}.

\begin{figure}[!t]
    \centering
     \includegraphics[width=0.6\linewidth]{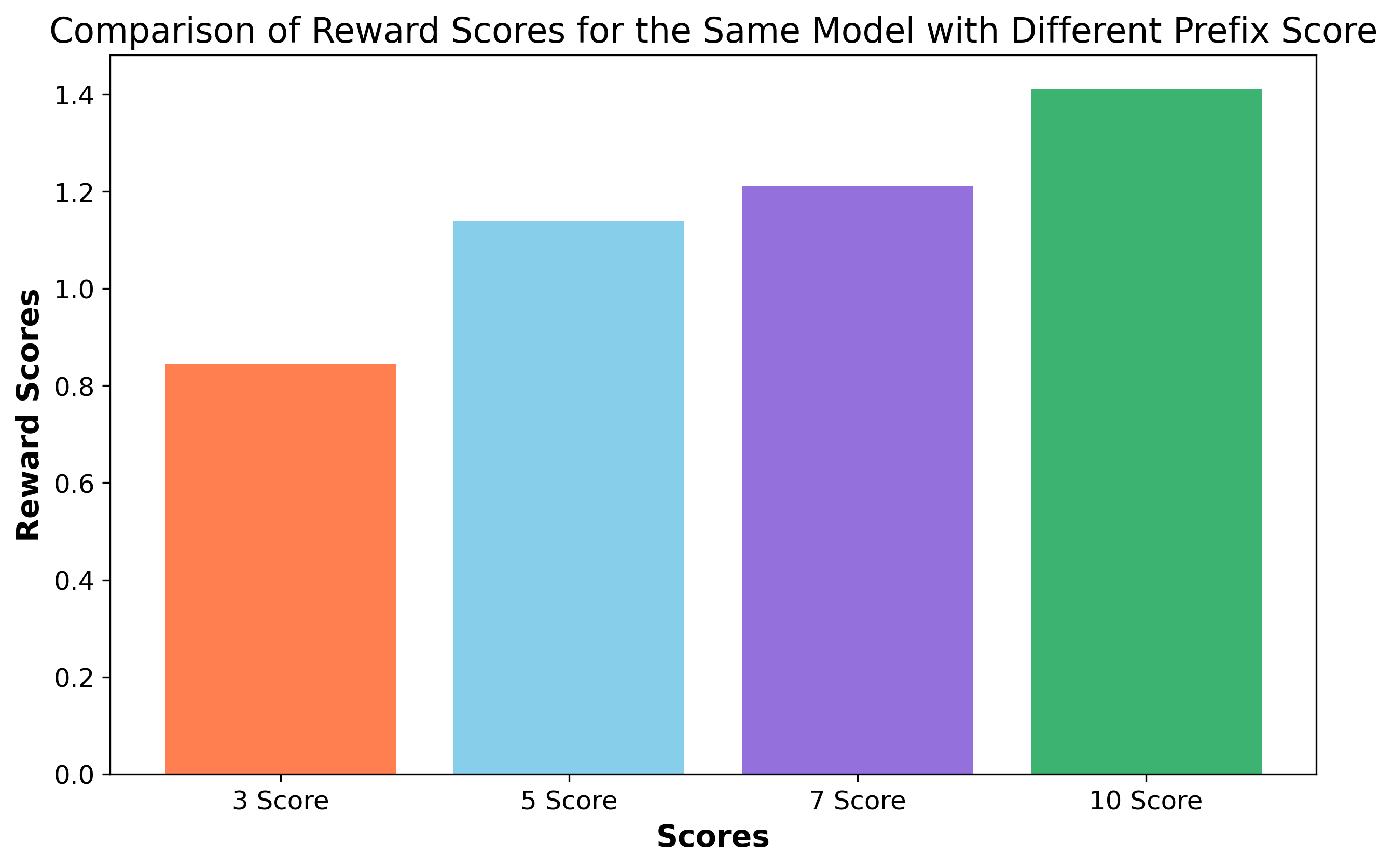}
    \caption{Average reward scores of responses generated with different prefix scores for the same model. The model used here is the reference model from the first iteration.}
        \label{fig:bar2}
\end{figure}

\begin{figure}[!t]
    \centering
     \includegraphics[width=0.6\linewidth]{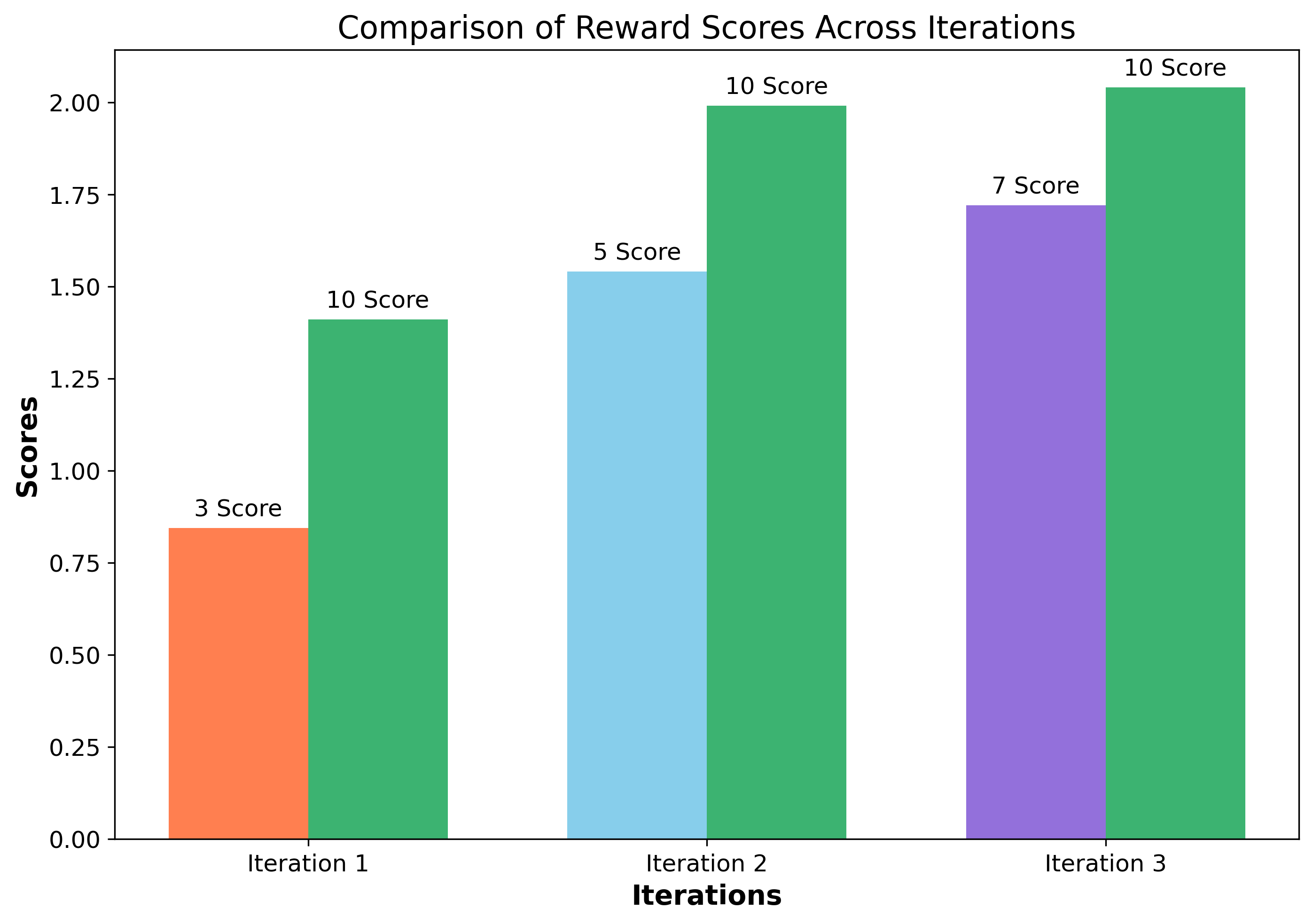}
    \caption{Average reward scores of chosen responses and rejected responses across different iterations. For each iteration, rejected responses are shown on the left and chosen responses on the right. The rejected scores for iterations 1, 2, and 3 are 3, 5, and 7, respectively.}
    \label{fig:bar3}
\end{figure}

\begin{table}[!t]
\centering
\caption{Comparison of results without arithmetic control and our method. The removal of arithmetic control refers to fixing the rejected score at 3. In our method, the rejected scores for each iteration are set to 3, 5, and 7, respectively.}
\label{tab:score}
\begin{tabular}{@{}lccc@{}}
\toprule
                            & \multicolumn{2}{c}{AlpacaEval 2.0} & MT-Bench \\
                            & LC  (\%)              & AC  (\%)               &          \\ \midrule
w/o arithmetic control-iter1 & 23.4            & 21.6           & 5.8      \\
w/o arithmetic control-iter2                          & 22.8            & 20.4           & 5.9      \\
w/o arithmetic control-iter3                            & 23.7            & 21.4           & 5.9      \\ \midrule
Ours-iter1                        & 23.4            & 21.6           & 5.8      \\
Ours-iter2                           & 24.2            & 23.5           & 6.0      \\
Ours-iter3                           & 25.9            & 24.9           & 6.0      \\ \bottomrule
\end{tabular}
\end{table}

\begin{table}[!t]
\centering
\caption{Comparison of inference results with the chosen prompt and our method.}
\label{tab:infer}
\begin{tabular}{@{}lccc@{}}
\toprule
                          & \multicolumn{2}{c}{AlpacaEval 2.0} & MT-Bench \\ \midrule
                          & LC  (\%)                & AC (\%)                &          \\ \midrule
SimPO                     & 21.5            & 20.1           & 6.0      \\
Inference + Chosen prefix & 22.5            & 21.9           & 5.5      \\ \midrule
Ours-iter1                      & 23.4            & 21.6           & 5.8      \\
Ours-iter2                         & 24.2            & 23.5           & 6.0      \\
Ours-iter3                          & 25.9            & 24.9           & 6.0      \\ \bottomrule
\end{tabular}
\end{table}

\section{Conclusion}
In this paper, we propose a novel only-prompting self-rewarding online preference optimization method for LLM alignment. Our method relies solely on generation capability to create preference datasets without needing a discriminator, addressing the issue of poor judgment capability in small models and significantly saving on inference and API costs. Additionally, we provide a mathematical proof that preference datasets generated with different prefix scores exhibit quality differences, which can be used for DPO training.
Furthermore, by applying fine-grained score control to the responses, we generate preference datasets of varying quality. In the later stages of training, we gradually reduce the quality gap between chosen and rejected responses, which forces the model to better align with complex human preferences. Extensive experiments on two widely used benchmarks, AlpacaEval 2.0 and MT-Bench, demonstrate the superiority of our method.
\newpage
\bibliography{iclr2025_conference}
\bibliographystyle{iclr2025_conference}

\end{document}